\documentclass[letterpaper, 10 pt, conference]{ieeeconf}  

\IEEEoverridecommandlockouts                              

\usepackage{graphicx}
\usepackage{epsfig} 
\usepackage{mathptmx}
\usepackage{amsmath} 
\usepackage{bm}

\usepackage{amssymb}  
\usepackage{amsthm}
\usepackage{siunitx}
\usepackage{wasysym}
\usepackage[mathcal]{euscript}
\usepackage{bbm}
\usepackage{color}
\usepackage{lipsum}
\usepackage[ruled,vlined,linesnumbered]{algorithm2e}
\usepackage[colorlinks=true,linkcolor=black,anchorcolor=black,citecolor=black,filecolor=black,menucolor=black,runcolor=black,urlcolor=black]{hyperref}
\usepackage{etoolbox}
\usepackage[table,xcdraw,dvipsnames]{xcolor}
\usepackage{multirow}
\usepackage{bbm}
\usepackage{accents} 

\usepackage{caption}
\usepackage{outlines}
\let\labelindent\relax
\usepackage{enumitem}
\setenumerate[2]{label=\alph*.}
\setenumerate[3]{label=\roman*.}

\makeatletter
\patchcmd{\@makecaption}
  {\scshape}
  {}
  {}
  {}
\makeatother

\usepackage[
    style=ieee,
    doi=false,
    isbn=false,
    url=false,
    eprint=false,
    backend=biber,
    natbib=true
    ]{biblatex}
    
\bibliography{references}

\title{\LARGE \bf
Provably-Safe Neural Network Training Using \\ Hybrid Zonotope Reachability Analysis
}

\author{Long Kiu Chung and Shreyas Kousik
\thanks{The authors are with the Department of Mechanical Engineering, Georgia Institute of Technology, Atlanta, GA.
\textbf{Corresponding author: } \texttt{lchung33@gatech.edu}.
\textbf{Website: }\href{https://saferoboticslab.me.gatech.edu/research/hybrid-zonotope-training}{https://saferoboticslab.me.gatech.edu/research/hybrid-zonotope-training}.
\textbf{GitHub: }\href{https://github.com/safe-robotics-lab-gt/hybZonoTraining}{https://github.com/safe-robotics-lab-gt/hybZonoTraining}.
}
}

\begin{document}
\newcommand{\edgar}[1]{{{\textcolor{red}{LKC: #1}}}}

\newtheorem{defn}{Definition}
\newtheorem{rem}[defn]{Remark}
\newtheorem{lem}[defn]{Lemma}
\newtheorem{prop}[defn]{Proposition}
\newtheorem{assum}[defn]{Assumption}
\newtheorem{ex}[defn]{Example}
\newtheorem{thm}[defn]{Theorem}
\newtheorem{cor}[defn]{Corollary}
\newtheorem{problem}[defn]{Problem}

\providecommand{\R}{\ensuremath \mathbb{R}}
\newcommand{\Natnum}{\ensuremath \mathbb{N}}
\newcommand{\Ppoly}{P}
\newcommand{\ReLUgraph}{H}
\newcommand{\inset}{Z}
\newcommand{\obs}{U}
\newcommand{\collider}{Q}

\newcommand{\regtext}[1]{\mathrm{\textnormal{#1}}}
\newcommand{\defemph}[1]{\emph{#1}}
\newcommand{\vc}[1]{\mathbf{#1}}
\newcommand{\surr}[1]{\tilde{#1}}

\newcommand{\st}{\regtext{ s.t. }}

\newcommand{\norm}[1]{\left\Vert#1\right\Vert}
\newcommand{\abs}[1]{\left\vert#1\right\vert}
\newcommand{\diag}[1]{\regtext{diag}\!\left(#1\right)}
\newcommand{\absvec}[1]{\regtext{abs}\!\left(#1\right)}
\newcommand{\interior}[1]{\regtext{int}\!\left(#1\right)}
\newcommand{\tp}{^\intercal}
\newcommand{\logn}{\regtext{ln}}
\newcommand{\ind}{\mathbbm{1}}
\newcommand{\bigo}{\mathcal{O}}
\newcommand{\convhull}{\regtext{conv}}
\newcommand{\abselem}{\regtext{abs}}

\newcommand{\idx}[1]{_{#1}}

\newcommand{\binlab}{\regtext{b}}
\newcommand{\conlab}{\regtext{c}}
\newcommand{\genlab}{\regtext{g}}
\newcommand{\scalelab}{\regtext{r}}
\newcommand{\odd}{\regtext{odd}}
\newcommand{\even}{\regtext{even}}

\newcommand{\bsub}{_{\binlab}}
\newcommand{\csub}{_{\conlab}}
\newcommand{\gsub}{_{\genlab}}
\newcommand{\rsub}{_{\scalelab}}

\newcommand{\opt}{^{*}}

\newcommand{\ndim}{n}
\newcommand{\mdim}{m}
\newcommand{\nvert}{\ndim_{\regtext{v}}}
\newcommand{\npoly}{\ndim_{\regtext{N}}}

\newcommand{\zono}{\mathcal{Z}}
\newcommand{\conzono}{\mathcal{CZ}}
\newcommand{\hyzono}{\mathcal{HZ}}
\newcommand{\scaledzono}{\mathcal{SHZ}}

\newcommand{\eye}{\mathbf{I}}
\newcommand{\Rmat}{\vc{R}}
\newcommand{\emptyarr}{[\,]}

\newcommand{\zeros}{\mathbf{0}}
\newcommand{\ones}{\mathbf{1}}
\newcommand{\xv}{\vc{\xs}}
\newcommand{\yv}{\vc{\ys}}
\newcommand{\sv}{\vc{s}}
\newcommand{\vv}{\vc{v}}
\newcommand{\uv}{\vc{u}}
\newcommand{\dv}{\vc{s}}
\newcommand{\paramv}{\vc{k}}

\newcommand{\alphas}{a}
\newcommand{\xs}{x}
\newcommand{\ys}{y}
\newcommand{\zs}{z}
\newcommand{\ctrls}{u}
\newcommand{\ts}{t}
\newcommand{\bitl}{k}
\newcommand{\is}{i}
\newcommand{\rs}{r}
\newcommand{\mus}{\mu}
\newcommand{\surrr}{\surr{\rs}}
\newcommand{\surrl}{\surr{\ell}}

\newcommand{\nn}{\vc{\xi}}
\newcommand{\Wt}{\vc{W}}
\newcommand{\bias}{\vc{w}}
\newcommand{\depth}{d}
\newcommand{\nneu}{\ndim_{\regtext{n}}}
\newcommand{\midset}{Y}
\newcommand{\Rx}[1]{\Rmat_{\xv\idx{#1}}}
\newcommand{\Rv}[1]{\Rmat_{\vv\idx{#1}}}

\newcommand{\Gen}{\vc{G}}
\newcommand{\Genc}{\Gen\csub}
\newcommand{\Genb}{\Gen\bsub}
\newcommand{\Acon}{\vc{A}}
\newcommand{\Aconc}{\Acon\csub}
\newcommand{\Aconb}{\Acon\bsub}
\newcommand{\ctr}{\vc{c}}
\newcommand{\bcon}{\vc{b}}
\newcommand{\coef}{\vc{z}}
\newcommand{\bcoef}{\coef\bsub}
\newcommand{\ccoef}{\coef\csub}
\newcommand{\surrzco}{\surr{\coef}_{\conlab 1}}
\newcommand{\surrzct}{\surr{\coef}_{\conlab 2}}
\newcommand{\surrzcth}{\surr{\coef}_{\conlab 3}}
\newcommand{\surrzb}{{\surr{\coef}_{\binlab}}}
\newcommand{\ncon}{\ndim\csub}
\newcommand{\nbin}{\ndim\bsub}
\newcommand{\ngen}{\ndim\gsub}

\newcommand{\nscale}{\ndim\rsub}

\newcommand{\dynmat}{\vc{C}}
\newcommand{\dynvec}{\vc{d}}
\newcommand{\nstate}{\ndim_{\xs}}
\newcommand{\nctrl}{\ndim_{\ctrls}}
\newcommand{\indomain}{X}
\newcommand{\insetnow}{X_{\ts}}
\newcommand{\insetnext}{X_{\ts+\Delta\ts}}

\newcommand{\hplane}{H}
\newcommand{\Hcon}{\vc{H}}
\newcommand{\hcon}{\vc{h}}

\newcommand{\nparam}{\ndim_{k}}
\newcommand{\paramD}{K}
\newcommand{\tf}{{\ts}_{\regtext{f}}}
\newcommand{\D}{S} 
\newcommand{\Dfun}{\ensuremath \mathbb{S}} 
\newcommand{\initset}{\indomain_0}
\newcommand{\Xgoal}{G}
\newcommand{\Xobs}{O}
\newcommand{\nnk}{\nn_{k}}
\newcommand{\nng}{\nn_{g}}
\newcommand{\err}{\vc{e}}
\newcommand{\HZT}{T}
\newcommand{\HZerr}{E}
\newcommand{\graphK}{\midset_{k}}
\newcommand{\outG}{\Ppoly_{g}}
\newcommand{\obst}{\obs_{\ts}}
\newcommand{\obstf}{\obs_{\tf}}
\newcommand{\robvol}{B}

\newcommand{\Gencin}{\Gen_{\conlab, \inset}}
\newcommand{\Genbin}{\Gen_{\binlab, \inset}}
\newcommand{\Aconcin}{\Acon_{\conlab, \inset}}
\newcommand{\Aconbin}{\Acon_{\binlab, \inset}}
\newcommand{\ctrin}{\ctr_{\inset}}
\newcommand{\bconin}{\bcon_{\inset}}
\newcommand{\nconin}{\ndim_{\conlab, \inset}}
\newcommand{\nbinin}{\ndim_{\binlab, \inset}}
\newcommand{\ngenin}{\ndim_{\genlab, \inset}}
\newcommand{\Gencobs}{\Gen_{\conlab, \obs}}
\newcommand{\Genbobs}{\Gen_{\binlab, \obs}}
\newcommand{\Aconcobs}{\Acon_{\conlab, \obs}}
\newcommand{\Aconbobs}{\Acon_{\binlab, \obs}}
\newcommand{\ctrobs}{\ctr_{\obs}}
\newcommand{\bconobs}{\bcon_{\obs}}
\newcommand{\nconobs}{\ndim_{\conlab, \obs}}
\newcommand{\nbinobs}{\ndim_{\binlab, \obs}}
\newcommand{\ngenobs}{\ndim_{\genlab, \obs}}
\newcommand{\Genccol}{\Gen_{\conlab, \collider}}
\newcommand{\Genbcol}{\Gen_{\binlab, \collider}}
\newcommand{\Aconccol}{\Acon_{\conlab, \collider}}
\newcommand{\Aconbcol}{\Acon_{\binlab, \collider}}
\newcommand{\ctrcol}{\ctr_{\collider}}
\newcommand{\bconcol}{\bcon_{\collider}}
\newcommand{\nconcol}{\ndim_{\conlab, \collider}}
\newcommand{\nbincol}{\ndim_{\binlab, \collider}}
\newcommand{\ngencol}{\ndim_{\genlab, \collider}}
\newcommand{\outset}{\Ppoly_{\depth}}
\newcommand{\Gencout}{\Gen_{\conlab, \outset}}
\newcommand{\Genbout}{\Gen_{\binlab, \outset}}
\newcommand{\Aconcout}{\Acon_{\conlab, \outset}}
\newcommand{\Aconbout}{\Acon_{\binlab, \outset}}
\newcommand{\ctrout}{\ctr_{\outset}}
\newcommand{\bconout}{\bcon_{\outset}}
\newcommand{\outgraph}{\midset_{\depth}}
\newcommand{\Gencog}{\Gen_{\conlab, \outgraph}}
\newcommand{\Genbog}{\Gen_{\binlab, \outgraph}}
\newcommand{\Aconcog}{\Acon_{\conlab, \outgraph}}
\newcommand{\Aconbog}{\Acon_{\binlab, \outgraph}}
\newcommand{\ctrog}{\ctr_{\outgraph}}
\newcommand{\bconog}{\bcon_{\outgraph}}

\maketitle

\begin{abstract}
Even though neural networks are being increasingly deployed in safety-critical control applications, it remains difficult to enforce constraints on their output, meaning that it is hard to \textit{guarantee} safety in such settings.
While many existing methods seek to \textit{verify} a neural network's satisfaction of safety constraints, few address how to correct an unsafe network.
The handful of works that extract a training signal from verification cannot handle non-convex sets, and are either conservative or slow.
To begin addressing these challenges, this work proposes a neural network training method that can encourage the \textit{exact} image of a non-convex input set for a neural network with rectified linear unit (ReLU) nonlinearities to avoid a non-convex unsafe region.
This is accomplished by reachability analysis with \textit{scaled hybrid zonotopes}, a modification of the existing hybrid zonotope set representation that enables parameterized scaling of non-convex polytopic sets with a differentiable collision check via mixed-integer linear programs (MILPs).
The proposed method was shown to be effective and fast for networks with up to 240 neurons, with the computational complexity dominated by inverse operations on matrices that scale linearly in size with the number of neurons and complexity of input and unsafe sets.
We demonstrate the practicality of our method by training a forward-invariant neural network controller for an affine dynamical system with a non-convex input set, as well as generating safe reach-avoid plans for a black-box dynamical system.
\end{abstract}
\section{Introduction}
Though neural networks have seen success in many domains, they are also well-known as ``black-box'' models, where the relationship between their inputs and outputs is not easily interpretable or analyzable due to non-linearity and high-dimensional parameterizations.
As such, it is very difficult to certify their \textit{safety} (e.g.\ satisfaction of constraints).
This limits the success in real-life training and deployment of deep reinforcement learning (RL) policies \cite{dulac2021challenges} and makes learned systems susceptible to adversarial attacks \cite{eykholt2018robust}, causing injuries and accidents.
In this paper, we present a method to enforce safety in neural networks by encouraging satisfaction of a collision-free constraint, which we applied to train forward-invariant control policies and synthesize safe motion plans for black-box systems.
An overview of our method is shown in Fig.~\ref{fig:front_figure}.

\begin{figure}[t]
\vspace*{1.7mm}
\centering
    \includegraphics[width=1\columnwidth]{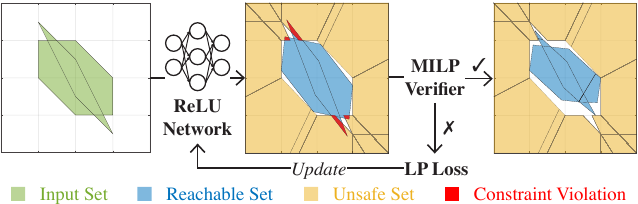}
\caption{
A flowchart of our method, using the example from Sec.~\ref{sec:nonconvex_fi}, where the goal is to make the closed-loop system forward-invariant.
Our method takes in a non-convex input set (green), then computes its \textit{exact} reachable set (blue) through the neural network.
Then, we formulate the reachable set's collision (red) with a non-convex unsafe set (yellow) as a loss function using a linear program (LP), which enables us to update the neural network's parameters via backpropagation.
The training stops once a mixed-integer linear program (MILP) verifies constraint satisfaction.
}\label{fig:front_figure}
\vspace*{-0.5cm}
\end{figure}

\subsection{Related Work}

We now review key approaches to enforce constraints on neural network, as well as relevant literature on \textit{hybrid zonotopes}, which we used in our method.

\subsubsection{Training with Constraints}\label{sec:soft_training}
Many existing works capture safety in neural networks by penalizing constraint violations on sampled \textit{points} during training \cite{gu2022review, liu2023constrained}.
While this approach is fast and easy to implement, it does not provide any safety guarantees beyond the training samples.
There are works that are capable of enforcing \textit{hard} constraints in neural networks by modifying the training process \cite{balestriero2023police}, but these can only handle affine constraints.

\subsubsection{Neural Network Verification}\label{sec:nn_verification}
A different approach is to certify safety with respect to a \textit{set} of inputs, typically by analyzing the image of the input set through a neural network \cite{ortiz2023hybrid, kochdumper2023open}.
However, most of these works only focus on \textit{verification}.
That is, they only answer the yes-no question of ``safe'' or ``unsafe'', with the aftermath of fixing an ``unsafe'' network left unexplored. 
As a result, one can only train via trial-and-error until the constraints have been satisfied, which can be slow and ineffective.

\subsubsection{Training with Verification}
A few methods combine Sec.~\ref{sec:soft_training} and Sec.~\ref{sec:nn_verification} by using set-based verification to identify counterexamples \cite{tran2020verification, tran2023verification, dai2021lyapunov}, which the training can penalize.
However, since they target only on violating \textit{points}, a ``whack-a-mole'' problem can often occur, where achieving safety on certain samples causes violation of others, trapping the algorithm in a perpetual loop.

To harvest the greatest amount of information from set-based verification, a handful of works formulated the penalty for constraint violation directly from the set representation.
However, they are either over-approximative, (which limits the space of discoverable solutions) \cite{harapanahalli2024certified, mirman2018differentiable, lin2020art, wong2018provable}, can only handle simple constraints \cite{dai2021lyapunov}, or has complexity that grows exponentially with the size of the neural network \cite{chung2021constrained, yang2022neural}.
To the best of our knowledge, none of the existing method in this category can directly handle non-convex input sets.

\subsubsection{Hybrid Zonotopes}
Recently, a non-convex polytopic set representation called hybrid zonotopes \cite{bird2023hybrid} was proposed.
Hybrid zonotopes are closed under many set operations \cite{bird2023hybrid, bird2021unions}, with extensive toolbox support \cite{koeln2024zonolab}.
They can also exactly represent the image \cite{ortiz2023hybrid} and preimage \cite{zhang2023backward} for a neural network with rectified linear units (ReLU) using basic matrix operations, with complexity scaling linearly with the network size.
However, existing works on hybrid zonotopes enforce control system safety either by formulating optimization problems without neural networks in the loop \cite{bird2021unions, bird2024set}, or only use hybrid zonotope for verification but not training \cite{ortiz2023hybrid, zhang2023backward, zhang2024hybrid, zhang2023reachability}.
In this paper, our contribution is extracting and using learning signals from neural network reachability analysis with hybrid zonotopes, enabling constraint satisfaction for non-convex input sets with exact and fast analysis.

\subsection{Contributions}
Our contributions are threefold:
\begin{enumerate}
    \item Given a non-convex input set and a non-convex unsafe region, we propose a differentiable loss function that, when trained on a ReLU neural network, encourages the image of the input set to avoid the unsafe region, the satisfaction of which can be verified by a mixed-integer linear program (MILP).
    Our key insight to achieving this is to introduce \textit{scaled hybrid zonotopes}, which enable parameterized scaling of hybrid zonotopes.
    \item We show that our method is fast and scales well with network size, dimension, and set complexity for networks with up to 240 neurons.
    The method significantly outperforms our prior work \cite{chung2021constrained}.
    \item We showcase the practicality of our method by synthesizing a \textit{forward-invariant} neural network controller for a non-convex input set, as well as safe motion plans for a parallel-parking vehicle subjected to black-box, \textit{extreme drifting} dynamics. 
\end{enumerate}
\section{Preliminaries}\label{sec:preliminaries}

We now introduce our notation for hybrid zonotopes and ReLU networks.

\subsection{Hybrid Zonotope}
A \textit{hybrid zonotope} $\hyzono(\Genc, \Genb, \ctr, \Aconc, \Aconb, \bcon) \subset \R^\ndim$ is a set parameterized by generator matrices $\Genc\in\R^{\ndim\times\ngen}$, $\Genb\in\R^{\ndim\times\nbin}$, a center $\ctr\in\R^{\ndim}$, constraint matrices $\Aconc\in\R^{\ncon\times\ngen}$, $\Aconb\in\R^{\ncon\times\nbin}$, and constraint vector $\bcon\in\R^{\nbin}$ on continous coefficients $\ccoef\in\R^{\ngen}$ and binary coefficients $\bcoef\in\{-1, 1\}^{\nbin}$ as follows \cite[Definition 3]{bird2023hybrid}:
\begin{align}\label{eq:def_hybzono}\begin{split}
    \hyzono(&\Genc, \Genb, \ctr, \Aconc, \Aconb, \bcon)\\ 
    = \{&\Genc\ccoef + \Genb\bcoef + \ctr \mid \Aconc\ccoef + \Aconb\bcoef = \bcon, \norm{\ccoef}_\infty\leq1, \\
    &\bcoef\in\{-1, 1\}^{\nbin}\}.
\end{split}\end{align}
We denote $\ngen$ as the number of continuous generators, $\nbin$ as the number of binary generators, and $\ncon$ as the number of constraints in a hybrid zonotope.

Consider a pair of hybrid zonotopes $\Ppoly_1 \subset \R^{\ndim_1}$ and $\Ppoly_2 \subset \R^{\ndim_2}$.
In this paper, we make use of their closed form expressions for 
intersections ($\Ppoly_1 \cap \Ppoly_2 = \{\xv \mid \xv \in \Ppoly_1, \xv \in \Ppoly_2\}$),
Cartesian products ($\Ppoly_1 \times \Ppoly_2 = \{\begin{bmatrix}
    \xv\tp & \yv\tp
\end{bmatrix}\tp \mid \xv \in \Ppoly_1, \yv \in \Ppoly_2\}$),
affine maps ($\Wt\Ppoly_1 + \bias = \{\Wt\xv + \bias\mid\xv\in\Ppoly_1\}$, for some mapping matrix $\Wt$ and vector $\bias$) \cite{bird2023hybrid, koeln2024zonolab},
Minkowski sums ($\Ppoly_1 \oplus \Ppoly_2 = \{\xv + \yv \mid \xv \in \Ppoly_1, \yv \in \Ppoly_2\}$) \cite{bird2023hybrid},
union ($\Ppoly_1 \cup \Ppoly_2 = \{\xv \mid \xv \in \Ppoly_1 \regtext{ or } \xv \in \Ppoly_2\}$), 
set differences ($\Ppoly_1\setminus\Ppoly_2 = \{\xv\mid\xv\in\Ppoly_1, \xv \notin \interior{\Ppoly_2}\}$, where $\interior{\cdot}$ is the interior) \cite{bird2021unions},
and intersections with hyperplanes ($\Ppoly_1 \cap \hplane = \{\xv \mid \xv \in \Ppoly_1, \xv \in \hplane\}$, where $\hplane=\{\xv\mid\Hcon\xv=\hcon\}$).

\subsection{ReLU Neural Network}
In this work, we consider a fully-connected, ReLU activated feedforward neural network $\nn:\R^{\ndim\idx{0}}\to\R^{\ndim\idx{\depth}}$, with output $\xv_\depth = \nn(\xv_0) \in \R^{\ndim\idx{\depth}}$ given an input $\xv_0 \in \R^{\ndim\idx{0}}$.
Mathematically,
\begin{subequations}
\begin{align}
    \vv\idx{\is} &= \Wt\idx{\is}\xv\idx{\is-1} + \bias\idx{\is},\\
    \xv\idx{\is} &= \max\left(\vv\idx{\is}, \zeros\right),\label{eq:nn_hiddenlayer}\\
    \xv\idx{\depth} &= \vv\idx{\depth} = \Wt\idx{\depth}\xv\idx{\depth-1} + \bias\idx{\depth},\label{eq:finallayer}
\end{align}
\end{subequations}
where $\depth\in\Natnum$, $\Wt\idx{\is} \in \R^{\ndim\idx{\is}\times \ndim\idx{\is-1}}$, $\bias\idx{\is} \in \R^{\ndim\idx{\is}}$, $\is=1,\cdots,\depth-1$, $\Wt\idx{\depth} \in \R^{\ndim\idx{\depth}\times \ndim\idx{\depth-1}}$, $\bias\idx{\depth} \in \R^{\ndim\idx{\depth}}$, and max is taken elementwise.
We denote $\Wt\idx{1}, \cdots, \Wt\idx{\depth}$ as \textit{weights}, $\bias\idx{1}, \cdots, \bias\idx{\depth}$ as \textit{biases},
$\depth$ as the \textit{depth},
and $\ndim\idx{\is}$ as the \textit{width} of the $\is^{\regtext{th}}$ layer of the network.
The function $\max(\cdot, \zeros_{\ndim\idx{\is}\times1})$ is known as an $\ndim\idx{\is}$-dimensional \textit{ReLU activation function}, and layers $1$ to $\depth-1$ are known as \textit{hidden layers}.

Consider an input hybrid zonotope $\inset\subset\R^{\ndim\idx{0}}$.
Then, its image $\outset\subset\R^{\ndim\idx{\depth}}$ for a ReLU network is exactly a hybrid zonotope \cite{ortiz2023hybrid, koeln2024zonolab}:
\begin{subequations}\label{eq:hybzono_forward}
\begin{align}
    \midset_0 &= \inset \times \ReLUgraph_{\ndim\idx{1}} \times \cdots \times \ReLUgraph_{\ndim\idx{\depth-1}},\\
    \midset_\is &= \midset_{\is-1}\cap\{\xv\mid(\Rv{\is} - \Wt\idx{\is}\Rx{\is-1})\xv=\bias\idx{\is}\},\\
    \outgraph &= \begin{bmatrix}
        \Rx{0}\\
        \Wt\idx{\depth}\Rx{\depth-1}
    \end{bmatrix}\midset_{\depth-1} + \begin{bmatrix}
        \zeros \\ \bias\idx{\depth}
    \end{bmatrix},\label{eq:hybzono_nngraph}\\
    \outset &= \begin{bmatrix}
        \zeros & \eye
    \end{bmatrix} \outgraph\label{eq:hybzono_nnout},
\end{align}
\end{subequations}
where $\is = 1, \cdots, \depth - 1$.
$\Rx{0} \in \R^{\ndim\idx{0}\times(\ndim\idx{0} + 2\ndim\idx{1} + \cdots + 2\ndim\idx{\depth-1})}$, $\Rx{\is}, \Rv{\is} \in \R^{\ndim\idx{\is}\times(\ndim\idx{0} + 2\ndim\idx{1} + \cdots + 2\ndim\idx{\depth-1})}$ are defined as:
\begin{subequations}
\begin{align}
    \Rx{0} &= \begin{bmatrix}
        \eye & \zeros
    \end{bmatrix},\\
    \Rx{\is} &= \begin{bmatrix}
        \zeros_{\ndim\idx{\is}\times(\ndim\idx{0} + 2\ndim\idx{1} + \cdots + 2\ndim\idx{\is - 1} + \ndim\idx{\is})} & \eye & \zeros
    \end{bmatrix},\\
    \Rv{\is} &= \begin{bmatrix}
        \zeros_{\ndim\idx{\is}\times(\ndim\idx{0} + 2\ndim\idx{1} + \cdots + 2\ndim\idx{\is - 1})} & \eye & \zeros
    \end{bmatrix},
\end{align}
\end{subequations}
and $\ReLUgraph_{\ndim\idx{\is}} \subset \R^{2\ndim\idx{\is}}$ is the graph of an $\ndim\idx{\is}$-dimensional ReLU activation function over a hypercube domain $\{\xv\mid-\alphas\ones\leq\xv\leq\alphas\ones\}$ with radius $\alphas$, which is exactly a hybrid zonotope \cite{zhang2023backward}:
\begin{align}\begin{split}
    \ReLUgraph_{\ndim\idx{\is}} =&\  \left\{\begin{bmatrix}
        \xv \\
        \max(\xv, \zeros)
    \end{bmatrix}\mid-\alphas\ones\leq\xv\leq\alphas\ones\right\},\\
    =&\ \hyzono\biggl(\begin{bmatrix}
        \eye\otimes\begin{bmatrix}
            -\frac{\alphas}{2} & -\frac{\alphas}{2} & 0 & 0
        \end{bmatrix}\\
        \eye\otimes\begin{bmatrix}
            0 & -\frac{\alphas}{2} & 0 & 0
        \end{bmatrix}
    \end{bmatrix}, \begin{bmatrix}
        -\frac{\alphas}{2}\eye \\
        \zeros
    \end{bmatrix}, \frac{\alphas}{2}\ones,\\
    &\ \eye\otimes\begin{bmatrix}
        \eye_2 & \eye
    \end{bmatrix}, \eye\otimes\begin{bmatrix}
        1 \\
        -1
    \end{bmatrix}, \ones\biggr),
\end{split}\end{align}
where $\otimes$ is the Kronecker product.
Note that \eqref{eq:hybzono_forward} holds as long as $\alphas$ is large enough \cite{ortiz2023hybrid}, and that $\outgraph$ is exactly the graph $\{\begin{bmatrix}
    \xv_0\tp & \nn({\xv_0})\tp
\end{bmatrix}\tp\mid\xv_0\in\inset\}$.
If $\inset$ has $\ngenin$ continuous generators, $\nbinin$ binary generators, and $\nconin$ constraints, then $\outset$ will have $\ngenin + 4\nneu$ continuous generators, $\nbinin + \nneu$ binary generators, and $\nconin + 3\nneu$ constraints \cite{zhang2023backward}, where $\nneu := \ndim\idx{1} + \cdots + \ndim\idx{\depth-1}$ denotes the \textit{number of neurons}.

\section{Problem Statement}\label{sec:problem_statement}
Our goal in this paper is to design a ReLU network training method such that the image of a given input set avoids some unsafe regions.
As per most other training methods, we assume that the depth and widths of the ReLU network are fixed, focusing only on updating the weights and biases (a.k.a.\ trainable parameters).
Mathematically, we aim to to tackle the following problem:
\begin{problem}[Training a Neural Network to Avoid Unsafe Regions]\label{prob:nn_avoid}
    Given an input set $\inset =\hyzono(\Gencin, \Genbin, \ctrin, \Aconcin, \Aconbin, \bconin)\subset \R^{\ndim_0}$, an unsafe region $\obs=\hyzono(\Gencobs, \Genbobs, \ctrobs, \Aconcobs, \Aconbobs, \bconobs)\subset \R^{\ndim_\depth}$, and a ReLU neural network $\nn$ with fixed depth $\depth$ and widths $\ndim_0, \cdots, \ndim_\depth$, find $\Wt_1, \cdots, \Wt_\depth, \bias_1, \cdots, \bias_\depth$ such that
    \begin{align}\label{eq:avoid_con}
        \collider:=\outset\cap\obs = \emptyset,
    \end{align}
    where $\outset = \{\nn(\xv_0)\mid\xv_0\in\inset\}$. 
    We call $\collider$ the collision set.
\end{problem}
A trivial solution would be to set $\Wt_\depth = \zeros$ and $\bias_\depth \notin \obs$, but this is not useful.
Moreover, trivial solutions may not exist when some of the weights and biases are locked to certain values and/or the ReLU network is subjected to additional structural constraints (such as those in Sec.~\ref{sec:app}).
Instead, we aim to design a \textit{differentiable loss function} such that \eqref{eq:avoid_con} can be achieved by updating the trainable parameters via backpropagation.
Doing so allows our method to integrate with other loss functions to achieve additional objectives and enables applications to more complicated problems.
\section{Proposed Method}\label{sec:methods}
The key component to our method is designing a loss function that quantifies how badly the reachable set of a hybrid zonotope is in collision.
To accomplish this, we first define a set representation that can grow/shrink a hybrid zonotope, design a loss function based on it, and explain how to differentiate the loss function to achieve safety.

\subsection{Scaled Hybrid Zonotopes}\label{sec:scaled_hybzono}
To mathematically represent the enlargement and shrinkage of a hybrid zonotope, we define a new set representation that we call a \textit{scaled hybrid zonotope}:
\begin{defn}[Scaled Hybrid Zonotope]\label{def:scaledzono}
    A scaled hybrid zonotope $\scaledzono(\Ppoly, \rs, \nscale) \subset \R^{\ndim}$ is parameterized by a hybrid zonotope $\Ppoly = \hyzono(\Genc, \Genb, \ctr, \Aconc, \Aconb, \bcon) \subset \R^{\ndim}$, a scaling factor $\rs\in\R_+$, and scaling index $\nscale\in\{0, \cdots, \ngen\}$ as:
    \begin{align}\label{eq:def_scaledzono}
    \begin{split}
        &\scaledzono(\Ppoly, \rs, \nscale)\\
        =& \{\Genc\ccoef + \Genb\bcoef + \ctr \mid \Aconc\ccoef + \Aconb\bcoef = \bcon, \norm{(\ccoef)_{1:\nscale}}_\infty\leq\rs, \\
    &\norm{(\ccoef)_{(\nscale+1):\ngen}}_\infty\leq1, \bcoef\in\{-1, 1\}^{\nbin}\}.
    \end{split}
    \end{align}
\end{defn}

We confirm that this representation enables scaling:
\begin{cor}[Scaling with Scaled Hybrid Zonotopes]\label{cor:hyzono_to_scaledzono}
    Let $\Ppoly\rsub = \scaledzono(\Ppoly, \rs, \nscale)$.
    For any $\nscale$, $\Ppoly\rsub \subseteq \Ppoly$ if $\rs \leq 1$, $\Ppoly\rsub = \Ppoly$ if $\rs = 1$, and $\Ppoly \subseteq \Ppoly\rsub$ if $\rs \geq 1$.
\end{cor}
\begin{proof}
    This follows from \eqref{eq:def_hybzono} and \eqref{eq:def_scaledzono}.
\end{proof}
\noindent We say that the scaled hybrid zonotope $\Ppoly\rsub$ is the hybrid zonotope $\Ppoly$ scaled by $\rs$ in the first $\nscale$ dimensions.

Note that a similar notion of \textit{scaled zonotopes} exists \cite{yang2021scalable}, where the scaling is performed on the generator matrix.
Instead, we scale the upper bound of the infinity norm, which enables optimization over the scaling of the hybrid zonotope, as will be shown in Sec.~\ref{sec:hybzono_diff}.

We now apply scaled hybrid zonotopes to Problem \ref{prob:nn_avoid}.
To proceed, we first cast the input set $\inset$ as $\inset\rsub = \scaledzono(\inset, \rs, \nscale)$ for some chosen $\nscale$.
We leave $\rs$ as a variable to be optimized.
Note that one can scale any subset of $\{1, \cdots, \ngenin\}$ by rearranging the columns in $\Gencin$ and $\Aconcin$.
While Corollary \ref{cor:hyzono_to_scaledzono} and our analysis hold true for any $\nscale$, we observed that different choices can affect the behavior of the scaling, which can impact the performance of our method.
We discuss one such scenario in Sec.~\ref{sec:nonconvex_fi}.

The following lemma and proposition construct the graph and image of a scaled hybrid zonotope for a ReLU network.
\begin{lem}[Operations on Scaled Hybrid Zonotopes]\label{lem:scaledzono_op}
    Consider hybrid zonotopes ${\Ppoly}_1$, $\Ppoly_2$, and collection of hyperplanes $\hplane$.
    If ${\Ppoly\rsub}_1 = \scaledzono(\Ppoly_1, \rs, \nscale)$, then:
    \begin{align}
        {\Ppoly\rsub}_1 \times \Ppoly_2 &= \scaledzono(\Ppoly_1\times\Ppoly_2, \rs, \nscale),\label{eq:cartprod_scaledzono}\\
        {\Ppoly\rsub}_1\cap\hplane &= \scaledzono(\Ppoly_1\cap\hplane, \rs, \nscale),\label{eq:intplane_scaledzono}\\
        \Wt{\Ppoly\rsub}_1 + \bias &=  \scaledzono(\Wt\Ppoly_1 + \bias, \rs, \nscale),\ \regtext{and}\label{eq:affinemap_scaledzono}\\
        {\Ppoly\rsub}_1 \cap \Ppoly_2 &= \scaledzono(\Ppoly_1 \cap \Ppoly_2, \rs, \nscale)\label{eq:int_scaledzono}.
    \end{align}
\end{lem}
\begin{proof}
    This follows from the derivations for these hybrid zonotope operations \cite{bird2023hybrid} not depending on the bound of the infinity norm and not incurring additional continuous coefficients with scaled bounds.
\end{proof}
\begin{prop}[Image and Graph of Scaled Hybrid Zonotopes for ReLU Networks]\label{prop:scaledzono_reach}
    Compute $\outgraph$ and $\outset$ with \eqref{eq:hybzono_forward} for $\inset$.
    Let $\inset\rsub = \scaledzono(\inset, \rs, \nscale)$.
    Then,
    \begin{align}
        {\outgraph}\rsub &= \scaledzono(\outgraph, \rs, \nscale)\ \regtext{and}\\
        {\outset}\rsub &= \scaledzono(\outset, \rs, \nscale),
    \end{align}
    where ${\outgraph}\rsub = \{\begin{bmatrix}
    \xv_0\tp & \nn({\xv_0})\tp
\end{bmatrix}\tp\mid\xv_0\in\inset\rsub\}$ and ${\outset}\rsub = \{\nn({\xv_0}) \mid \xv_0\in\inset\rsub\}$ is the graph and image of $\inset\rsub$ for $\nn$.
\end{prop}
\begin{proof}
    This follows from Lemma \ref{lem:scaledzono_op}, since \eqref{eq:hybzono_forward} consists of only the operations \eqref{eq:cartprod_scaledzono}, \eqref{eq:intplane_scaledzono}, \eqref{eq:affinemap_scaledzono}, and is valid for any bounded polytopic representations \cite{ortiz2023hybrid}.
\end{proof}

Proposition \ref{prop:scaledzono_reach} implies that we can quickly compute the scaled image and graph of a scaled input set by simply changing $\rs$.
Moreover, from its proof, this property also holds when the image or graph is further subject to any of the operations in Lemma \ref{lem:scaledzono_op}.
For example, since $\collider = \outset\cap\obs$ from \eqref{eq:avoid_con}, we have $\collider\rsub := {\outset}\rsub \cap \obs = \scaledzono(\collider, \rs, \nscale)$.
We leverage these properties for our following analyses.

\subsection{Hybrid Zonotope Emptiness Check}\label{sec:hybzono_empty}
Before constructing a loss function for training, we first need a way to check if $\eqref{eq:avoid_con}$ is true.
From \eqref{eq:hybzono_forward}, we can compute $\collider$ as a hybrid zonotope $\hyzono(\Genccol, \Genbcol, \ctrcol, \Aconccol, \Aconbcol, \bconcol)$ (with $\ngencol = \ngenin + \ndim\idx{0} + 4\nneu + \ngenobs$ continuous generators, $\nbincol = \nbinin + \nneu + \nbinobs$ binary generators, and $\nconcol = \nconin + \ndim\idx{0} + 3\nneu + \ndim\idx{\depth}$ constraints, where $\ngenobs$, $\nbinobs$, and $\nconobs$ are the number of continuous generators, binary generators, and constraints for $\obs$ respectively).
To check whether $\collider$ is empty, existing methods would formulate a feasibility MILP with $\ngencol$ continuous variables and $\nbincol$ binary variables \cite{bird2023hybrid}:
\begin{align}\label{eq:empty_milp_naive}\begin{split}
    \regtext{find}\ & \ccoef, \bcoef,\\
    \st & \Aconccol\ccoef + \Aconbcol\bcoef = \bconcol,\\
    & \norm{\ccoef}_\infty\leq1,\\
    & \bcoef\in\{-1, 1\}^{\nbin},
\end{split}\end{align}
which is infeasible iff $\collider=\emptyset$.
Note that \eqref{eq:empty_milp_naive} is NP-complete \cite{achterberg2020presolve}.
However, \eqref{eq:empty_milp_naive} is a feasibility program and provides no measure of the degree of emptiness of $\collider$.
The optimizers also do not exist when \eqref{eq:empty_milp_naive} is infeasible (i.e.\ the image is in collision with the unsafe region).
Thus, it is unclear how a loss function can be derived from \eqref{eq:empty_milp_naive} to drive $\collider$ empty.

Instead, using scaled hybrid zonotopes, we formulate the following MILP with one more continuous variable than \eqref{eq:empty_milp_naive} to enable a differentiable emptiness check:
\begin{thm}[Scalable Emptiness Check]
    Given a hybrid zonotope $\collider = \hyzono(\Genccol, \Genbcol, \ctrcol, \Aconccol, \Aconbcol, \bconcol)$ and $\nscale$.
    Consider the following MILP with the same constraints as $\collider\rsub = \scaledzono(\collider, \rs, \nscale)$:
    \begin{align}\label{eq:empty_milp}\begin{split}
        \min\ & \rs,\\
        \st & \Aconccol\ccoef + \Aconbcol\bcoef = \bconcol,\\
        & \norm{(\ccoef)_{1:\nscale}}_\infty\leq\rs,\\
        & \norm{(\ccoef)_{(\nscale+1):\ngencol}}_\infty\leq 1,\\
        & \bcoef\in\{-1, 1\}^{\nbincol},
    \end{split}\end{align}
    where $\rs\in\R_+$, $\ngencol$ and $\nbincol$ are the number of continuous and binary generators in $\collider$ (and $\collider\rsub$).
    Suppose $\rs\opt$ is the optimal value of \eqref{eq:empty_milp}.
    Then, $\collider = \emptyset$ iff $\rs\opt > 1$.
\end{thm}
\begin{proof}
    This follows directly from the definition of hybrid zonotopes in \eqref{eq:def_hybzono}.
\end{proof}

Since $\nscale$ of $\collider\rsub$ corresponds only to continuous coefficients within $\inset\rsub$, from Corollary \ref{cor:hyzono_to_scaledzono}, $\rs\opt$ can be intuitively understood as \textit{the maximum amount $\inset$ can decrease in size (or the minimum amount $\inset$ can increase in size) such that its image $\outset$ is still in collision with $\obs$}.
By construction, \eqref{eq:empty_milp} is feasible as long as $\collider \neq \emptyset$ (i.e.\ the image is in collision).
Thus, we have a measure in the degree of collision, with which we can construct a loss function to encourage $\collider$ to be empty.
We illustrate this concept in Fig.~\ref{fig:before_after_training}.

\begin{figure}[t]
\vspace*{1.7mm}
\centering
    \includegraphics[width=1\columnwidth]{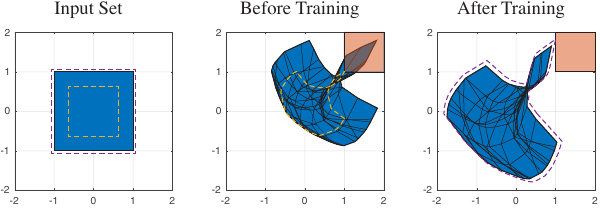}
\caption{
Results of our method in Sec.~\ref{sec:experiments} with 1 hidden layer of width 30.
$\inset$ and $\outset$ (input and image) are labelled in blue.
${\inset}\rsub$ and ${\outset}\rsub$ with $\rs=\rs\opt$ (scaled input and image) are shown for before (yellow) and after (purple) training.
This illustrates the concept that our method is minimizing the amount the input set needs to be shrunk (or maximizing the amount input set needs to be enlarged) before the image \textit{just} touches the unsafe set.
}
\label{fig:before_after_training}
\vspace*{-0.5cm}
\end{figure}

We note that \cite{chung2021constrained} employs a similar tactic to measure and train on the ``emptiness'' of constrained zonotopes \cite{scott2016constrained}.
In that work, the constraint by $\rs$ is placed on all continuous coefficients on $\collider$, which makes $\rs\opt$ less interpretable and loses flexibility in choosing a better $\nscale$.

\subsection{Loss Function to Encourage Emptiness}\label{sec:hybzono_diff}
We now construct a loss function which, when minimized, makes $\collider$ empty.
Since $\collider = \emptyset$ iff $\rs\opt > 1$, we can construct the loss function as $\ell = 1 - \rs\opt$, such that when $\ell$ is negative, we must have $\collider = \emptyset$.

To minimize $\ell$ using backpropagation, we need, from chain rule,  $\frac{\partial\ell}{\partial\rs\opt}$, $\frac{\partial\rs\opt}{\partial\Aconccol}$, $\frac{\partial\rs\opt}{\partial\Aconbcol}$, $\frac{\partial\rs\opt}{\partial\bconcol}$, $\frac{\partial\Aconccol}{\partial\Wt_1}$, $\cdots$, $\frac{\partial\bconcol}{\partial\Wt_\depth}$, and $\frac{\partial\Aconccol}{\partial\bias_1}$, $\cdots$, $\frac{\partial\bconcol}{\partial\bias_\depth}$.
While $\frac{\partial\ell}{\partial\rs\opt}$, $\frac{\partial\Aconccol}{\partial\Wt_1}$, $\cdots$, $\frac{\partial\bconcol}{\partial\bias_\depth}$ can be readily obtained from automatic differentiation \cite{paszke2019pytorch}, obtaining $\frac{\partial\rs\opt}{\partial\Aconccol}$, $\frac{\partial\rs\opt}{\partial\Aconbcol}$, and $\frac{\partial\rs\opt}{\partial\bconcol}$ involves differentiation through the MILP.
Since a MILP's optima can remain unchanged under small differences in its parameters, its gradient can be 0 or non-existent, which are uninformative \cite{hu2024two}.
Instead, consider the following convex relaxation of \eqref{eq:empty_milp}:
\begin{align}\label{eq:empty_lp}\begin{split}
    \min\ & \surrr - \mus(\ones\logn(\begin{bmatrix}
        \surrzco\tp & \surrzct\tp & \surrzcth\tp & \surrzb\tp & \surrr & \sv\tp
    \end{bmatrix}\tp)),\\
    \st & \Aconccol \begin{bmatrix}
        \surrzco - \surrzct\\
        2\surrzcth - \ones
    \end{bmatrix} + \Aconbcol(2\surrzb - \ones) = \bconcol,\\
    &\begin{bmatrix}
        \surrzco - \surrzct - \surrr\ones\\
        \surrzct - \surrzco - \surrr\ones\\
        \surrzcth\\
        \surrzb
    \end{bmatrix} + \sv = \begin{bmatrix}
        \zeros_{\nscale\times1}\\
        \zeros_{\nscale\times1}\\
        \ones_{(\ngencol-\nscale)\times1}\\
        \ones
    \end{bmatrix},
\end{split}\end{align}
where $\surrr \in \R_+$, $\surrzco \in \R^{\nscale}_+$, $\surrzct \in \R^{\nscale}$, $\surrzcth \in \R^{\ngencol-\nscale}$, $\surrzb \in \R^{\nbincol}_+$, $\sv \in \R^{\nscale+\ngencol+\nbincol}_+$, $\mus \in \R_+$ is the cut-off multiplier from the solver \cite{mandi2020interior}, and $\logn(\cdot)$ is applied elementwise.
Essentially, \eqref{eq:empty_lp} is the standard linear program (LP) form of \eqref{eq:empty_milp} with log-barrier regularization and without the integrality constraints, which be obtained by replacing $\rs$ with $\surrr$, $\ccoef$ with $\begin{bmatrix}
    (\surrzco - \surrzct)\tp & 2\surrzcth\tp - \ones
\end{bmatrix}\tp$, and $\bcoef$ with $2\surrzb - \ones$ (such that all constraints are non-negative), and introducing slack variable $\sv$ (such that inequality constraints become equality constraints).

The optimization problem \eqref{eq:empty_lp} can be solved quickly using solvers such as IntOpt \cite{mandi2020interior}.
Moreover, if ${\surrr}\opt$ is the optimal value of \eqref{eq:empty_lp}, $\frac{\partial{\surrr}\opt}{\partial\Aconc}$, $\frac{\partial{\surrr}\opt}{\partial\Aconb}$, and $\frac{\partial{\surrr}\opt}{\partial\bcon}$ can be obtained by differentiating the Karush-Kuhn-Tucker (KKT) conditions of \eqref{eq:empty_lp}, for which we refer the readers to \cite[Appendix B]{hu2024two} for the mathematical details.
Not only are these gradients well-defined, easily computable, and informative, but they have also been shown to outperform other forms of convex relaxation in computation speed and minimizing loss functions derived from MILPs \cite[Appendix E]{hu2024two}.

Therefore, instead of the loss function $\ell$, we propose to backpropagate with respect to a surrogate loss function $\surrl = 1 - \surrr\opt$, where $\surrr\opt$ is the optimal value of \eqref{eq:empty_lp}.
Since \eqref{eq:empty_lp} is the relaxation of \eqref{eq:empty_milp}, we have $\surrr\opt \leq \rs\opt$ and therefore $\ell \leq \surrl$ as $\mus \to 0$.
Thus for a small enough $\mus$, $\surrl < 0$ is a sufficient condition for $\collider = \emptyset$.
In practice, $\mus$ should not be set too small for better convergence to the MILP's optimal solution \cite{hu2024two}, and $\collider = \emptyset$ can be achieved much earlier before $\surrl < 0$.
Thus, we solve \eqref{eq:empty_milp_naive} or use any other neural network verification tool (see Sec.~\ref{sec:nn_verification}) in between some iterations of training with $\surrl$ to check whether \eqref{eq:avoid_con} has been achieved.
In other words, our approach is \textit{modular} to existing neural network verification frameworks.
Once \eqref{eq:avoid_con} has been verified as true, the training is complete and Problem \ref{prob:nn_avoid} has been solved.

\section{Experiments}\label{sec:experiments}
We now assess the scalability of our method with different network sizes.
We also compare our results with \cite{chung2021constrained} to assess our contribution.
All experiments were performed using Python on a desktop computer with a 24-core i9 CPU, 32 GB RAM, and an NVIDIA RTX 4090 GPU.

\subsection{Experiment Setup and Method}
We use the experiment setup in \cite{chung2021constrained} to benchmark our method.
First, we pretrain ReLU networks of different sizes (listed in Table \ref{table:experiments}) to approximate $\vc{f}(\xv) = \begin{bmatrix}
        (\xv)_1^2 + \sin((\xv)_2)&
        (\xv)_2^2 + \sin((\xv)_1)
    \end{bmatrix}\tp$ using mean squared error (MSE) loss and the Adam optimizer \cite{kingma2014adam} in PyTorch \cite{paszke2019pytorch}.

We define the input set as $\inset = \hyzono\left(\eye, [], \zeros, [], [], []\right)$ and the unsafe set as $\obs = \hyzono\left(0.5\eye, [], \begin{bmatrix}
        1.5&
        1.5
    \end{bmatrix}\tp, [], [], []\right)$ such that the image of $\inset$ would intersect with $\obs$ after pretraining, as shown in Fig.~\ref{fig:before_after_training}.
We note that $\inset$ and $\obs$ are convex for fair comparison with \cite{chung2021constrained}, which can only handle convex sets.
We will showcase our method against non-convex sets and neural network with multiple layers in Sec.~\ref{sec:app}.

Given the pretrained network, we begin training to obey the safety constraint.
In each training iteration, we use IntOpt \cite{mandi2020interior} with $\mus=0.1$ and \cite{hu2024two} to compute the value and gradient of the loss function $\surrl$ with $\nscale=2$ and $\alphas=50$.
Then, we use PyTorch \cite{paszke2019pytorch} with the Adam optimizer \cite{kingma2014adam} (with a learning rate of $0.02$) to update the trainable parameters in the network.
Every 5 iterations, we use Gurobi \cite{gurobi2021gurobi} to solve the MILP in \eqref{eq:empty_milp_naive} to check the emptiness of $\collider$.
We are successful in solving Problem \ref{prob:nn_avoid} if $\collider=\emptyset$, at which point we terminate the training instead of updating the parameters.

For \cite{chung2021constrained}, we implemented the method using the same optimizer and learning rate.
To ensure fairness, we do not include the objective loss and only add the constraint loss when it is positive.
We terminate the training once the constraint loss has reached zero.
We ran each experiment only once due to the excessive computation time of \cite{chung2021constrained}.

\begin{table}[t]
\vspace*{1.7mm}
\captionsetup{font=small}
\centering
\begin{tabular}{r|r|r|r|r}
\multicolumn{1}{c|}{Hidden} & \multicolumn{1}{c|}{Total} & \multicolumn{1}{c|}{Training} & \multicolumn{1}{c|}{MILP Time} & \multicolumn{1}{c}{Total}\\
\multicolumn{1}{c|}{Layer} & \multicolumn{1}{c|}{Time} & \multicolumn{1}{c|}{Time} & \multicolumn{1}{c|}{(\unit{s/}5 Itera-} &\multicolumn{1}{c}{Itera-}\\
\multicolumn{1}{c|}{Widths} & \multicolumn{1}{c|}{(\unit{s})} & \multicolumn{1}{c|}{(\unit{s/}Iteration)} & \multicolumn{1}{c|}{tions)} & \multicolumn{1}{c}{tions}\\
\hline
\multicolumn{5}{c}{Ours}\\
\hline
$(10)$ & $0.081$ & $0.003\pm0.001$ & $0.004\pm0.001$ & 25\\
$(20)$ & $0.111$ & $0.011\pm0.002$ & $0.007\pm0.003$ & 10\\
$(30)$ & $0.107$ & $0.025\pm0.003$ & $0.008\pm0.000$ & 5\\
$(60)$ & $2.767$ & $0.140\pm0.005$ & $0.026\pm0.001$ & 20\\
$(120)$ & $11.046$ & $1.205\pm0.023$ & $0.101\pm0.000$ & 10\\
$(240)$ & $47.855$ & $11.885\pm0.180$ & $0.317\pm0.000$ & 5\\
$(120, 120)$ & $47.444$ & $11.755\pm0.144$ & $0.423\pm0.000$ & 5\\
$(80, 80, 80)$ & $48.722$ & $11.832\pm0.262$ & $1.396\pm0.000$ & 5\\
\hline
\multicolumn{5}{c}{\cite{chung2021constrained}}\\
\hline
$(10)$ & $3.727$ & $0.373\pm0.016$ & N/A & 10\\
$(20)$ & $1{,}597.947$ & $319.589\pm1.987$ & N/A & 5\\
$(30)$ & Timeout & Timeout & N/A & N/A
\end{tabular}
\caption{Comparison of \cite{chung2021constrained} and our method's computation time needed to train the image of an input set for a ReLU network out of collision with an unsafe region.
For our method, we separately report the time needed to verify constraint satisfaction via solving a MILP, as it can be replaced with other verification methods.
}
\label{table:experiments}
\vspace*{-0.5cm}
\end{table}

\subsection{Hypotheses}\label{sec:exp_hypo}
Since the complexity in representing $\collider$ scales with the number of neurons linearly for our method \cite{ortiz2023hybrid} and exponentially for \cite{chung2021constrained}, we expect our method to significantly outperform \cite{chung2021constrained}, especially on larger networks.
We expect three operations to dominate the computation time: solving for the relaxed LP \eqref{eq:empty_lp}, which has a complexity of $\bigo(\bitl\sqrt{\ngencol+\nbincol+\nscale})$ ($\bitl$ is the bit length of the input data) \cite{wright1997primal}, solving for the MILP \eqref{eq:empty_milp_naive}, which has a worst-case complexity of $\bigo(2^{\nbincol})$ \cite{achterberg2020presolve}, and solving for the gradient of \eqref{eq:empty_lp}, which has a complexity of $\bigo((\ngencol+\nbincol+\nscale)^3)$ due to multiple matrix inverse operations \cite{hu2024two}.

\subsection{Results and Discussion}
We report the computation time of our experiments in Table \ref{table:experiments} and visualize a result in Fig.~\ref{fig:before_after_training}.
All images have been successfully driven out of the unsafe region, except \cite{chung2021constrained} for networks with more than 30 neurons, which failed to compute the images in a reasonable time.

As expected, our method scales far better than \cite{chung2021constrained} with the size of the neural network.
Surprisingly, solving for the MILP \eqref{eq:empty_milp_naive} did not take excessive time despite having a bad worst-case complexity, possibly due to further optimization from within Gurobi \cite{gurobi2021gurobi}.
The most time-consuming component of our method comes from computing the gradient of \eqref{eq:empty_lp}, which takes up $38.5\%$ of the training time for $\ndim=10$ and $97.9\%$ of the training time for $\ndim=240$.
To remedy this, we tried using a least-squares methods instead of directly solving for the matrix inverses, but did not observe any speed-up.
This implies that it may be more beneficial to increase the frequency of verification for problems with larger networks to reduce the overall computation time.

\section{Applications}\label{sec:app}
In this section, we demonstrate the utility of our method on problems in decision and control of robotic systems.
We used the same computing hardware and optimization packages as in Sec.~\ref{sec:experiments}.

\subsection{Forward Invariance for Non-Convex Safe Regions}\label{sec:nonconvex_fi}

A forward-invariant system is a well-known notion to certify infinite-time safe behavior.
However, to the best of our knowledge, only \cite{harapanahalli2024certified} has shown to be able to handle systems with neural network controllers in the loop---though it can only handle \textit{convex} sets.
Instead, we seek to use our method to learn forward-invariant neural network controllers for non-convex regions for the first time.

\subsubsection{Problem Statement}
Formally, we tackle the following:
\begin{problem}[Training a Forward-Invariant Controller]
    Consider an affine dynamical system $\xv_{\ts + \Delta\ts} = \dynmat \begin{bmatrix}
        \xv_{\ts}\tp & \uv_{\ts}\tp
    \end{bmatrix}\tp + \dynvec$, where $\xv_{\ts} \in \R^{\nstate}$ and $\uv_{\ts} \in \R^{\nctrl}$ are the states and control inputs of the current timestep, $\dynmat \in \R^{\nstate\times(\nstate+\nctrl)}$ and $\dynvec\in\R^{\nstate}$ characterizes the affine dynamics, and $\xv_{\ts+\Delta\ts} \in \R^{\nstate}$ is the states at the next timestep.
    Given an input set $\insetnow\subset\R^{\nstate}$, which is also the target safe set, train a neural network controller $\nn:\R^{\nstate}\to\R^{\nctrl}$ such that
    \begin{align}\label{eq:fi_con}
        \insetnext := \dynmat\midset + \dynvec \subseteq \insetnow,
    \end{align}
    where $\midset = \{\begin{bmatrix}
    \xv_{\ts}\tp & \nn({\xv_{\ts}})\tp
\end{bmatrix}\tp\mid\xv_{\ts}\in\insetnow\}$.
\end{problem}

\subsubsection{Method}
We can cast \eqref{eq:fi_con} as Problem \ref{prob:nn_avoid} by setting $\collider = \insetnext \cap (\indomain \setminus \insetnow)$ for some hyperrectangle $\indomain \supset \insetnow$, since driving $\insetnext$ within $\insetnow$ implies that $\insetnext$ has been driven out of $\indomain \setminus \insetnow$.

Note that formulating the unsafe region as $\indomain \setminus \insetnow$ would make \eqref{eq:empty_milp} a sufficient (but not necessary) condition for \eqref{eq:fi_con}, since collision may occur only at the boundary but not the interior of the unsafe region.
If $\insetnow$ is full-dimensional, we can overcome this by instead checking for the collision between $\insetnext$ and $\indomain \setminus (\insetnow \oplus \hyzono(\epsilon\eye, [], \zeros, [], [], []))$ for some small $\epsilon > 0$ when solving for \eqref{eq:empty_milp}, which slightly shrinks the unsafe set to make sure $\collider$ is full-dimensional when non-empty.

\begin{rem}
    In principle, we can also synthesize forward-invariant controllers for \textit{neural network dynamics} (using Proposition \ref{prop:scaledzono_reach}) and even general nonlinear dynamics (using techniques such as \cite{althoff2008reachability}).
    However, we note that not all choices of $\insetnow$ can be made forward-invariant (see \cite[Example 1]{harapanahalli2024certified}), and it is unclear how one can construct a ``forward-invariant-able'' $\insetnow$ for learned or general dynamics.
    To not overstate our contribution, we opted to demonstrate our method only on affine dynamical systems.
\end{rem}

\subsubsection{Setup}
We demonstrate our method on a double integrator with $\dynmat = \begin{bmatrix}
    \begin{bmatrix}1 & 0.1 & 0\end{bmatrix}\tp & \begin{bmatrix}0 & 1 & 0.1\end{bmatrix}\tp 
\end{bmatrix}\tp$ and $\dynvec = \zeros$.
We set the safe region as:
\begin{align}
\begin{split}
    \insetnow=&\hyzono(\\
    &\begin{bmatrix}
        0 & -1 & \zeros_{1\times3} & -1 & -0.75 & \zeros_{1\times2}\\
        -1 & 1 & \zeros_{1\times3} & 1 & 1.5 & \zeros_{1\times2}
    \end{bmatrix}, \begin{bmatrix}
        0.75\\
        -1.5
    \end{bmatrix},\begin{bmatrix}
        1 \\ -1
    \end{bmatrix}, \\
    &\begin{bmatrix}
        0 & 0 & 0 & 0 & 0 & 1 & 0 & 1 & 0\\
        0 & 0 & 0 & 0 & 0 & 0 & 0.75 & 0 & 0.75\\
        -1 & 1 & 1 & 0 & 0 & 0 & 0 & 0 & 0\\
        0 & 1 & 0 & 1 & 0 & 0 & 0 & 0 & 0\\
        1 & -1 & 0 & 0 & 1 & 0 & 0 & 0 & 0
    \end{bmatrix},\begin{bmatrix}
        -1.5\\-1\\0.75\\1\\0.25
    \end{bmatrix}\\
    &\begin{bmatrix}
        0.5&0.5&0.75&1&0.25
    \end{bmatrix}\tp),
\end{split}
\end{align}
which is obtained by the union of two convex regions that can be rendered forward-invariant \cite{harapanahalli2024certified} and reducing the number of generators and constraints using \cite{bird2023hybrid}.
We also have $\indomain = \hyzono(150\eye, [], \zeros, [], [], [])$, which was set to be much larger than $\insetnow$ to prevent gradient descent from driving $\insetnext$ entirely out of $\indomain$.
The resulting unsafe set $\indomain \setminus \insetnow$ has 84 continuous generators, 22 binary generators, and 60 constraints.
These sets are visualized in Fig.~\ref{fig:front_figure}.

Due to the simplicity of the system, we chose a neural network controller with 1 hidden layer of width 3, pretrained with a bad policy $\uv_{\ts} = \begin{bmatrix}
    -2 & -1
\end{bmatrix}\tp \xv_{\ts}$ using MSE loss and the Adam optimizer \cite{kingma2014adam} such that $\insetnext\nsubseteq\insetnow$.
We then apply our method with $\mus = 0.1$, $\nscale=5$, and $\alphas = 1{,}000$, using the Adam optimizer \cite{kingma2014adam} with a learning rate of $0.001$.
We check for \eqref{eq:fi_con} every $5$ iterations.

\subsubsection{Results and Discussion}
The results of our demo is shown in Fig.~\ref{fig:front_figure}.
A forward-invariant controller was successfully trained after 625 iterations.
The total computation time is $12.558 ~\unit{s}$, with training taking $0.017\pm0.002~\unit{s}$ per iteration and verification taking $0.016\pm0.005~\unit{s}$ per 5 iterations.

Although the double integrator system is simple, the problem is very difficult due to the nonlinearity of the neural network and the non-convexity of the complement set.
As such, we observed our method's performance to be highly sensitive to changes in hyperparameters.
For example, na\"ively choosing $\nscale=9$ would result in $\scaledzono(\insetnow, \rs, \nscale)=\emptyset\ \forall \rs < 1$, which gives bad gradient information.
On the other hand, choosing too high of a learning rate or using more na\"ive optimizers such as stochastic gradient descent (SGD) also prevented the training from converging to a safe solution.

\subsection{Reach-Avoid for Black-Box Dynamical Systems}\label{sec:neuralparc}

A common framework to ensure safety and liveness of robots is the \textit{reach-avoid} problem, where a robot must reach a goal while avoiding obstacles.
This is an especially difficult problem when the robot is \textit{black-box}, meaning that the dynamics are difficult to model analytically, but a dataset of interactions with the system is readily available.

Our recent work NeuralPARC \cite{chung2024guaranteed} addresses this problem by using neural networks to model the robot's behavior, then leverages backward reachability analysis to provide reach-avoid guarantees for black-box systems traveling through narrow gaps.
In this section, we will show how our hybrid zonotope training method can be applied to the same problem setting, and compare our results with those of NeuralPARC.

\subsubsection{Problem Statement}
Formally, we tackle the following:
\begin{problem}[Black-Box Reach-Avoid Problem]\label{prob:neuralparc}
    Consider translation-invariant trajectories of a robot with parameterized policies $\xv(\ts) = \vc{g}(\paramv, \dv, \ts) + \xv_0$, where $\xv\in\indomain\subset\R^{\nstate}$ is the position, $\xv_0 \in \initset \subset \indomain$ is the initial position, $\paramv\in\paramD\subset\R^{\nparam}$ represents \textit{trajectory parameters} (e.g.\ controller gains, desired goal position), $\ts \in [0, \tf]$ is time, $\tf \in \R_{+}$ is the final time, and $\dv(\cdot)\in\Dfun:=\{\phi:[0, \tf]\to\D\}$ is the disturbance function.
    We assume $\vc{g}$ to be continuous but having an unknown analytic expression, although we are allowed to observe its output by providing $\paramv, \dv, \ts$ offline.

    Then, given obstacles $\Xobs \subset \indomain$, a goal set $\Xgoal \subset \indomain$, and the initial set $\initset$, find a policy $\nnk:\R^{\nstate}\to\R^{\nparam}$ such that:
    \begin{align}
        \vc{g}(\nnk(\xv_0), \dv, \ts) + \xv_0 &\nsubseteq \Xobs,\label{eq:obs_avoid}\\
        \vc{g}(\nnk(\xv_0), \dv, \tf) + \xv_0 &\subseteq \Xgoal,\label{eq:goal_reach}
    \end{align}
    for all $\xv_0 \in \initset$, $\ts \in [0, \tf]$, and $\dv(\cdot)\in\Dfun$.
\end{problem}
\begin{rem}
    Unlike NeuralPARC, our method does not require the trajectories to be translation-invariant, and $\xv(\ts)$ can be states other than position.
    We formulated the problem as such only to perform a fair comparison.
\end{rem}

\subsubsection{Method}
We begin by training a neural network $\nng: \R^{\nparam + 1} \to \R^{\nstate}$ to approximate $\vc{g}$, with features $(\paramv, \ts)$ and labels $(\vc{g}(\paramv, \dv, \ts))$ using MSE loss and the Adam optimizer \cite{kingma2014adam} with uniform sampling on $\paramv, \ts, \dv$ from $\paramD, [0, \tf], \D$.
Then, we approximate the modeling error with sampling:
\begin{align}
    \err = \max\{\abselem(\nng(\paramv, \ts) - \vc{g}(\paramv, \dv, \ts))\mid\paramv\in\paramD, \dv(\cdot)\in\Dfun, \ts\in[0, \tf]\},
\end{align}
where $\abselem$ is applied elementwise.
See \cite[Section IV.D]{chung2024goal} for a discussion on the validity of the approach.

We can now formulate Problem \ref{prob:neuralparc} as Problem \ref{prob:nn_avoid}.
We first construct the \textit{time-parameterized forward reachable tube} (FRT) as $\outG$ by:
\begin{align}
    \HZT &= \hyzono\{0.5\tf, [], 0.5\tf, [], [], []\},\\
    \graphK &= \left\{\begin{bmatrix}
        \xv_0 \\
        \nnk(\xv_0)
    \end{bmatrix}\mid\xv_0\in\initset\right\}\times\HZT,\\
    \outG &= \begin{bmatrix}
        \zeros & \zeros & 1 & \zeros\\
        \eye_{\nstate} & \zeros & \zeros & \eye_{\nstate}
    \end{bmatrix}\left\{\begin{bmatrix}
        \xv_0 \\ \paramv \\ \ts \\ \nng(\paramv, \ts)
    \end{bmatrix}\mid\begin{bmatrix}
        \xv_0 \\ \paramv \\ \ts
    \end{bmatrix}\in\graphK\right\}.
\end{align}
Here, $\outG$ represents the future time and states of $\initset$ predicted by $\nng$ under the policy $\nnk$, and is exactly $\{\begin{bmatrix}
    \ts & \xv\tp_\ts
\end{bmatrix}\tp\mid\xv_\ts = \xv_0 + \nng(\nnk(\xv_0), \ts), \xv_0 \in \initset, \ts \in [0, \tf]\}$.
Importantly, it preserves the scaled hybrid zonotope properties from Lemma \ref{lem:scaledzono_op}.

Then, we can construct the unsafe set $\obs$ as:
\begin{align}
    \HZerr &= \hyzono\{\diag{\err}, [], \zeros, [], [], []\},\\
    \obst &= \HZT\times(\Xobs \oplus \robvol \oplus \HZerr),\label{eq:unsafe_time}\\
    \obstf &= \hyzono\{0, [], \tf, [], [], []\}\times((\indomain\setminus\Xgoal)\oplus\HZerr),\label{eq:unsafe_final}\\
    \obs &= \obst \cup \obstf,
\end{align}
where $\robvol$ is the polytopic overapproximation of the robot's circular volume \cite{chung2024goal, chung2024guaranteed}.
In essence, \eqref{eq:unsafe_time} and \eqref{eq:unsafe_final} encodes the constraints \eqref{eq:obs_avoid} and \eqref{eq:goal_reach} respectively by Minkowski summing with the modeling error (see \cite[Lemma 3]{chung2024guaranteed}) and applying Cartesian product with the corresponding time of the constraints.
We can now write \eqref{eq:avoid_con} as $\collider = \outG \cap \obs$.

\subsubsection{Demo Setup}
We compare our method with NeuralPARC on a parallel-parking vehicle under extreme drifting dynamics with the same $\vc{g}$, $\paramD$, $\Dfun$, $\tf$, and $\robvol$ as \cite{chung2024goal, chung2024guaranteed}.
We set $\indomain = \hyzono(50\eye, [], \zeros, [], [], [])$, $\initset = \hyzono(0.1\eye, [], \zeros, [], [], [])$, and $\Xgoal$ and $\Xobs$ as shown in Fig.~\ref{fig:drifting_fig} to mimic the setup in \cite{chung2024goal}.

\begin{figure}[t]
\vspace*{1.7mm}
\centering
    \includegraphics[width=1\columnwidth]{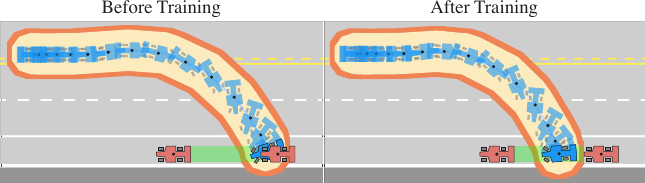}
\caption{
Illustration of the forward reachable tube for a drifting vehicle before (left) and after (right) training with our method.
We show the goal (green), the obstacles (red), and the forward reachable tube of the neural network model (blue) summed with the robot's volume (yellow) and further summed with the modeling error (orange).
A representative timelapse is shown for each picture.
The vehicle crashed with the obstacles before training, but is guaranteed to reach and avoid after applying our method.
}
\label{fig:drifting_fig}
\vspace*{-0.5cm}
\end{figure}

We first pretrain a policy $\nnk$ with 1 hidden layer of width 20 to approximate a bad motion plan $\paramv = [10, \frac{7}{36}\pi]\tp$ (the center of $\paramD$) such that the trajectories start off colliding with the obstacles.
With a trajectory model $\nng$ of 4 hidden layers of width 6, we applied our method with $\mus=0.1$, $\nscale = 2$, and $\alphas = 40$ using the Adam optimizer \cite{kingma2014adam} with a learning rate of 0.001.
We check whether $\collider=\emptyset$ every 5 iterations.

For NeuralPARC, we use the same problem setup with timestep $\Delta\ts = 0.1$ and an initial seed of $\paramv = [10, \frac{7}{36}\pi]\tp$.
100 points were sampled in each mode when searching for safe trajectory parameters.
The trajectory model has the same size as $\nng$ and is trained on the same set of data.

\subsubsection{Results and Discussion}
The results of our method are shown in Fig.~\ref{fig:drifting_fig}.
For our method, a safe policy was found after $13.742~ \unit{s}$ with 80 iterations.
On average, training took $0.166\pm0.006~\unit{s}$ per iteration and verification took $0.039\pm0.002~\unit{s}$ per 5 iterations.
For NeuralPARC, computing the full backward reach-avoid set took $9.526~\unit{s}$ for a total of 75 modes.
On average, finding a safe trajectory parameter took $0.601\pm0.089~\unit{s}$ per mode.

While the results show both methods to be effective, a direct comparison is difficult because of their slightly different objectives and susceptibility to hyperparameter changes.
In general, since NeuralPARC reasons with backward reachability, it is better suited for when the agent does not know where it should start to achieve safety.
Since our method reasons with forward reachability, there may not exist safe solutions for our chosen $\initset$ at all.
On the other hand, NeuralPARC scales linearly with the number of timesteps and the number of modes in the neural network, whereas our method scales linearly with the number of neurons and does not scale with timesteps.
Thus, our method is more suitable for trajectories with longer horizons or that require a larger network to model.
Finally, since our method has a smaller inference time to obtain the safe trajectory parameter (less than $0.001~\unit{s}$), our method is better-suited for scenarios where rapid update to the motion plans is desired, such as when the agent is traveling very quickly.
\section{Conclusion}\label{sec:conclusion}

This work proposes a new training method for enforcing constraint satisfaction by extracting learning signals from neural network reachability analysis using hybrid zonotopes.
The method is exact and can handle non-convex input sets and unsafe regions, and is shown to be fast and scalable to different network sizes and set complexities, significantly outperforming our previous work \cite{chung2021constrained}.
We belive this opens new approaches to synthesizing forward-invariant controllers and safe motion plans.

\subsubsection*{Limitations}
We observed four major limitations of our method.
First, our method is sensitive to the choice of hyperparameters.
While we currently rely on intuition about the problem setting to select these parameters, we believe development of a more rigorous procedure is possible, similar to many other works in deep learning \cite{feurer2019hyperparameter}.
Second, the computation time of our method is bottlenecked by matrix inverse computations.
Instead, future work could explore alternate gradient computation strategies such as differentiation of the homogeneous self-dual (HSD) formulation \cite{mandi2020interior} instead of the KKT conditions.
Third, our method cannot verify whether the given problem is solvable, which is significant for the examples in Sec.~\ref{sec:app}.
As such, future work could explore using gradient information to train the input set and the neural network simultaneously \cite{yang2024lyapunov}.
Last, our method is limited to fully-connected networks with ReLU activation functions.
However, we are optimistic that future work can extend to convolution neural networks (CNN) and recurrent neural networks (RNN) similar to \cite{tran2020verification, tran2023verification}.

\renewcommand{\bibfont}{\normalfont\footnotesize}
{\renewcommand{\markboth}[2]{}
\printbibliography}

\end{document}